\documentclass{article}
\usepackage{fullpage}
\usepackage{amsmath, amssymb,bm}
\usepackage{hyperref}
\usepackage{mathtools}
\usepackage{amsthm}
\usepackage{cite}
\usepackage{microtype}
\usepackage{enumitem}
\usepackage{mathrsfs}
\usepackage{xcolor}
\usepackage{tikz}
\usepackage{pgfplots}

\theoremstyle{definition}
\newtheorem{theorem}{Theorem}
\newtheorem{cor}{Corollary}
\newtheorem{lemma}{Lemma}

\newtheorem{remark}{Remark}

\newtheorem{condition}{Condition}

\newcommand{\bbR}{\mathbb{R}}
\newcommand{\bbN}{\mathbb{N}}

\newcommand{\Id}{\mathrm{I}}

\DeclareMathOperator{\gtr}{tr}

\newcommand{\normal}{\mathsf{N}}

\newcommand{\E}{\mathbb{E}}
\renewcommand{\P}{\mathbb{P}}

\DeclarePairedDelimiter\bkt{[}{]}     
\makeatletter
\newcommand{\@exstar}[1]{\E \bkt*{#1}}
\newcommand{\@exnostar}[2][]{\E \bkt[#1]{#2}}
\newcommand{\ex}{\@ifstar\@exstar\@exnostar}
\makeatother

\makeatletter
\newcommand{\@prstar}[1]{\P \bkt*{#1}}
\newcommand{\@prnostar}[2][]{\P \bkt[#1]{#2}}
\newcommand{\pr}{\@ifstar\@prstar\@prnostar}
\makeatother

\DeclareMathOperator{\cov}{\mathsf{Cov}}

\pgfplotsset{compat=newest} 

\let\originalleft\left
\let\originalright\right
\renewcommand{\left}{\mathopen{}\mathclose\bgroup\originalleft}
\renewcommand{\right}{\aftergroup\egroup\originalright}

\mathtoolsset{showonlyrefs}

\allowdisplaybreaks

\newcommand{\target}{\mu}
\newcommand{\Np}{\tilde{N}}
\newcommand{\Xf}{X}
\newcommand{\Xb}{\bar{X}}

\newif\ifisit
\isittrue

\newif\ifcomments
\commentstrue

\title{Information-Theoretic Proofs for Diffusion Sampling}
\author{Galen Reeves and Henry D. Pfister}

\begin{document}

\maketitle

\begin{abstract}
This paper provides an elementary, self-contained analysis of diffusion-based sampling methods for generative modeling.  In contrast to existing approaches that rely on continuous-time processes and then discretize, our treatment works directly with discrete-time stochastic processes and yields precise non-asymptotic convergence guarantees under broad assumptions.  The key insight is to couple the sampling process of interest with an idealized comparison process that has an explicit Gaussian–convolution structure.  We then leverage simple identities from information theory, including the I-MMSE relationship, to bound the discrepancy (in terms of the Kullback-Leibler divergence) between these two discrete-time processes.  In particular, we show that, if the diffusion step sizes are chosen sufficiently small and one can approximate certain conditional mean estimators well, then the sampling distribution is provably close to the target distribution. Our results also provide a transparent view on how to accelerate convergence by using additional randomness in each step to match higher-order moments in the comparison process. 
\end{abstract}

%\tableofcontents

\section{Introduction}

Diffusion-based sampling methods have emerged as powerful tools for machine learning applications~\cite{Sohl-icml15,song:2019generative,Ho-neurips20,Nichol-icml22,Rombach-cvpr22,Saharia-neurips22}. The high-level idea behind these methods is to define a stochastic process that transforms a sequence of samples from an easy-to-sample distribution (e.g., an isotropic  Gaussian) into a sample from a target distribution on a high-dimensional space (e.g., a natural image)~\cite{Sohl-icml15,Ho-neurips20}.

The theoretical justification for these methods typically follows a two-stage argument~\cite{Song-iclr21}. First, one specifies a continuous-time process, called a diffusion, that models the underlying distribution of interest. Then, one argues that this process can be simulated accurately by a discrete-time process to generate approximate samples from the target distribution. 

This paper presents a simple and intuitive \emph{discrete-time} proof that explains the effectiveness of diffusion-based sampling methods. The origin of this work lies in the authors' desire to provide an elementary presentation of diffusion models suitable for first-year graduate students. Focusing directly on discrete-time stochastic processes, we derive precise non-asymptotic guarantees under very general assumptions. 
Along the way, this approach unearths interesting connections between diffusion modeling and the celebrated I-MMSE relationship from information theory \cite{guo:2005a}, which provides a link between mutual information (the 'I') and the minimum mean-squared error (MMSE) in additive Gaussian noise models. 

We note that novelty in this work lies more in the path chosen for the presentation than in the mathematical details of the individual steps, which may have appeared in some form previously in the literature.

\subsection{Overview of Main Results}

Consider the problem of sampling from a target distribution $\target$ on $\bbR^d$. In practice this distribution may be known exactly, or it may only be described implicitly by a set of samples. Many popular sampling methods generate a process that can be represented by
\begin{align}
    Z_k = Z_{k-1}  +  \delta_k f_k(Z_{k-1}) +  \sqrt{\delta_k} \Np_k   \label{eq:Zk}
\end{align}
starting from $Z_0 = 0$, where $\delta_1, \delta_2, \dots$ are step sizes, $f_1, f_2, \dots$ are functions $f_k \colon\bbR^d \to \bbR^d$, and $\Np_1, \Np_2, \dots$ are independent standard Gaussian vectors. 

If the functions are linear then this is a classical  first-order autoregressive Gaussian process. The more interesting setting for modern applications arises when the functions are nonlinear and the resulting distribution is non-Gaussian. One canonical choice for $f_k$ is given by the mapping from $z\in \bbR^d$ to the conditional mean\footnote{
There is an affine mapping between the conditional mean and the \emph{score function}, i.e., the gradient of the log density of $Y_{k-1}$.} of $X \sim \mu$ given an observation $Y_{k-1} = z$ in Gaussian noise; see \eqref{eq:Ykalt} below.

For the purpose of theoretical analysis, we introduce a  ``comparison process'' defined by 
\begin{align}
    Y_k =  Y_{k-1}  + \delta_{k} X + \sqrt{\delta_k}N_k \label{eq:Yk}
\end{align}
starting from 
$Y_0 = 0$, where $X \sim \target$ is independent of the Gaussian noise sequence $\{N_k\}$. In contrast to \eqref{eq:Zk}, the distribution of this process is easy to describe. Indeed, by summing the increments and defining $t_k \coloneqq \delta_1 + \dots + \delta_k$, this process can be expressed equivalently as
\begin{align} 
    Y_k = t_k X + W_k  , \qquad  W_k \coloneqq  \sum_{i=1}^k \sqrt{\delta_i} N_i . \label{eq:Ykalt}
\end{align}
Since $W_1, W_2, \dots$ is a zero-mean Gaussian process with covariance $\cov(W_k,W_m) = \min\left\{ t_k, t_m \right\}\Id$, the marginal distribution of $Y_k$ satisfies
\begin{align}
    \mathsf{Law}(t_k^{-1} Y_k)  = \target \ast \normal( 0, t_k^{-1} \Id) ,
\end{align}
where $\normal( m, K)$ denotes a Gaussian measure with mean $m$ and covariance $K$ and $\ast$ denotes the convolution of measures. For sufficiently large $t_k$, a sample from this distribution is often considered a suitable proxy for a sample from $\target$. 

Of course, the process in \eqref{eq:Yk} is not a viable sampling strategy because its implementation requires a sample from the target distribution. In contrast, the process in \eqref{eq:Zk} only requires samples from the standard Gaussian distribution.  This paper focuses on the divergence between these two processes and computes an exact expression for 
\begin{align}
\Delta_n & \coloneqq D\big(\mathsf{Law}(Y_1,\dots, Y_n)  \, \| \,  \mathsf{Law}(Z_1, \dots, Z_n)\big),
\end{align}
where $D(P\,\|\,Q)$ denotes the Kullback-Leibler divergence\footnote{In this paper, all logarithms are natural and all quantities of information are expressed in nats.} (or relative entropy) between distributions $P$ and $Q$.
The following theorem gives a bound on $\Delta_n$ that depends only on the covariance of $\target$,  the step sizes $\{\delta_k\}$ and how well each $f_k$ approximates the conditional mean estimator of $X$ given $Y_{k-1}$. \begin{theorem}\label{thm:DYZ}
Assume that $\target$ has finite second moments and $\ex{\|f_k(Y_{k-1})\|^2} < \infty$ for all $k = 1, \dots, n$. Then,  
% \ifisit
% \begin{align}
% \Delta_n
%     & \le  \frac{\delta_\mathrm{max}}{2} \gtr( \cov(X)) \\
%     & \qquad 
%     +\sum_{k=1}^{n} \frac{\delta_k}{2}  \ex*{ \|f_k(Y_{k-1}) - \ex{ X \mid Y_{k-1}} \|^2},
% \end{align}
% \else
\begin{align}
\Delta_n
    & \le  \frac{\delta_\mathrm{max}}{2} \gtr( \cov(X)) 
    +\sum_{k=1}^{n} \frac{\delta_k}{2}  \ex*{ \|f_k(Y_{k-1}) - \ex{ X \mid Y_{k-1}} \|^2},
\end{align}
% \fi
where $\delta_\mathrm{max} \coloneqq \max \{ \delta_1, \dots, \delta_n\}$. 
\end{theorem}

\begin{remark} Theorem~\ref{thm:DYZ} shows that the distributions of  $\{Y_k\}$ and $\{Z_k\}$  can be made arbitrarily close provided that the step sizes $\{\delta_k\}$ are small enough and the functions $\{f_k\}$ accurately approximate the conditional mean estimator defined by the comparison process.  For example, suppose that the goal is to produce an approximate sample from $\target \ast \normal(0, T^{-1} \Id) $ for given value  $T > 0$. Assuming $f_k$ is equal to the conditional mean estimator and using uniform step sizes $\delta_k = T/n$ leads to 
\begin{align*}
    \Delta_n \le \frac{T \gtr(\cov(X))}{2n}. \label{eq:simple_delta_n}
\end{align*}
In this way, the problem of producing a sample has been reduced to the problem of computing a sequence of conditional-mean estimates at different noise levels.
\end{remark}

\begin{remark} The result is also ``dimension-free'' in that neither the assumptions nor the bound depend explicitly on $d$. Thus, for example, this result extends to distributions on the infinite-dimensional Hilbert space of square summable sequences. 
\end{remark}

\begin{remark}
By virtue of Pinsker's inequality, a bound on $\Delta_n$ implies a bound on the total variation distance:
\begin{align}
\mathsf{TV}\big(\mathsf{Law}(Y_1,\dots, Y_n) ,   \mathsf{Law}(Z_1, \dots, Z_n) \big) \le \sqrt{ \tfrac{1}{2} \Delta_n}  .  
\end{align}
Of course, this implies that $\mathsf{TV}\left(\mathsf{Law}(Y_n) ,   \mathsf{Law}(Z_n) \right)$ is upper bounded by the same quantity.
\end{remark}

An elementary proof of Theorem~\ref{thm:DYZ} is presented in Section~\ref{sec:proof}, and the intuition behind the sampling scheme and connections with continuous-time models are discussed in Section~\ref{sec:intuition}. Additional results that follow as natural consequences of our general approach are stated in Section~\ref{sec:further_results}.  In particular:
\begin{itemize}
\item Theorem~\ref{thm:alpha_bound}
 provides a ``dimension-free'' bound that allows $T$ to grow superlinearly in $n$, where the exact dependence is determined by the high-SNR scaling of the mutual information for the target distribution in an additive Gaussian noise channel.
 \item Theorem~\ref{thm:moments} generalizes to the setting where $f_k$ returns random values from a distribution that approximates the conditional distribution of $X$ given $Y_{k-1}$ in \eqref{eq:Yk}.  It is shown that if the moments are matched up to order $m \in \bbN$, then $\Delta_n$ decreases at the rate $n^{-m}$. For example, matching the mean results in the $n^{-1}$ dependence implied by Theorem~\ref{thm:DYZ} and matching second moments gives a rate $n^{-2}$. 
\end{itemize}

\begin{figure}
    \centering
\usepgfplotslibrary{fillbetween}
%\begin{center}
\begin{tikzpicture}
    \begin{axis}[
    	width = 5in,
	height = 3in,
        axis x line=bottom, axis y line=left,
        legend style = {draw opacity=1},
        ylabel style={yshift=-1.2cm},
        ylabel={MSE},
        ymin=0, ymax=1.1,
        xmin=0, xmax=5.1,
        samples=100,
        domain=0:5,
        xtick={0,1,2,3,4,5},
        xticklabels={$t_0$, $t_1$, $t_2$, $t_3$, $t_4$, $t_5$},
        ytick={0,1},
        yticklabels ={$0$, $\gtr(\cov(X))$},
        legend cell align={left},
        ylabel style={rotate=-90},
%          xlabel near ticks,
    ]
    
        % Function plot
        \addplot[name path = f, blue, very thick, domain=0:5, samples=100 ] {1/(1+x)}; % node[pos=0.3, below] { $M(t)$};
        \addlegendentry{$M(t)$} 
         \addplot[name path = axis , forget plot] coordinates {(0,0) (5,0)};
         \addplot[red!15, area legend] fill between[of=axis and axis, soft clip={domain=0:5}];
        \addlegendentry{$= 2 \Delta_5$}           
         \addplot[blue!15, area legend] fill between[of=f and axis, soft clip={domain=0:5}];
          \addlegendentry{$=  2 I(X; Y_{5})$}
        
       \pgfplotsinvokeforeach{0,1,2,3,4} 
       {
      \addplot[ name path = g#1, red, very thick ] coordinates {(#1,{1/(1+#1)}) ({#1+1},{1/(1+#1)})};
      \addplot[red!30, opacity=0.4] fill between[of=f and g#1, soft clip={domain=#1:{#1+1}}];
       }
    \end{axis}
\end{tikzpicture}
\caption{Our results provide an exact connection between the divergence $\Delta_n$ and the mutual information $I(X;Y_{n})$ in terms of the MMSE function $M(t) \coloneqq \E \|X-\E[X\mid\sqrt{t}X+N]\|^2$ for the target distribution $\target$.  Assuming each $f_k$ is the conditional mean estimator, Lemma~\ref{lem:DYZexact} shows that  $\Delta_n$ is equal to one half of the (red) area between the integral of $ M(t)$ and its left Riemann approximation. The I-MMSE relation states that the mutual information $I(X;Y_{n})$ is equal to one half the (blue) area under the MMSE function. The upper bound in Theorem~\ref{thm:DYZ} follows from the fact that the sum of the red areas cannot exceed $\delta_\mathrm{max} \gtr(\cov(X))$. 
    }
    \label{fig:Delta_Integral}
\end{figure}
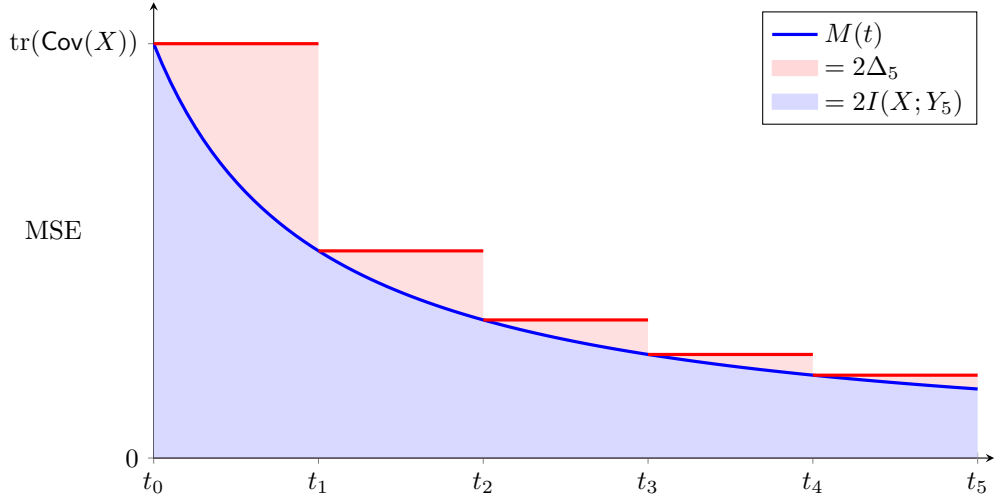

\subsection{Background and Related Work}

Diffusion sampling has become very popular for generative models due to its amazing performance on collections of digital images from the Web.
An important component of these models is that the generation is conditional (e.g., on a text prompt) and this corresponds to the functions $f_k$ depending on that prompt.
In 2022, the Imagen text-to-image model based on conditional diffusion sampling was released and widely celebrated~\cite{Saharia-neurips22}.
The results were clearly better than the groundbreaking DALLE-1 text-to-image model, released in 2021, which is based instead on transformers that generate visual tokens~\cite{Ramesh-icml21}.
This encouraged many researchers to focus on diffusion sampling.
For example, the DALLE-2 model from 2022 is based on diffusion sampling~\cite{ramesh-arxiv22}.

But, the current interest in generative diffusion models actually traces back to a 2015 paper~\cite{Sohl-icml15} that was improved by a sequence of follow-on papers~\cite{Ho-neurips20,Nichol-icml22,Rombach-cvpr22}. In particular, the first papers worked in pixel space~\cite{Ho-neurips20,Nichol-icml22}. Later, significant speedups were achieved by performing the diffusion in latent space (e.g., the diffusion process operates in the latent space defined by a model trained for image recognition and reconstruction)~\cite{Rombach-cvpr22}. Significant gains were also seen with larger language models for prompts, hierarchical generation, and upsampling~\cite{Saharia-neurips22}.

Theoretically, this early work led to analyses based on stochastic differential equations~\cite{Song-iclr21} and connections to an older idea known as stochastic localization~\cite{Eldan-gafa13,Chen-focs22}.
More recently, these ideas have been connected to information theory~\cite{Alaoui-it22,Montanari-arxiv23,Kong-iclr23,Kong-iclr24}. Recent work on convergence rates includes \cite{lee:2023convergence,chen:2023sampling,chen:2023improved,benton:2024nearly,li:2024towards,li:2024accelerating} and acceleration methods proposed in  \cite{wu2024:stochastic,li:2024provable}.

\section{Proof of Theorem~\ref{thm:DYZ}} 
\label{sec:proof}

\begin{lemma}
\label{lem:MarkovProperty}
The process $\{Y_k\}$ defined in \eqref{eq:Yk} is a Markov chain. 
\end{lemma}

At its core, the Markov property is a  consequence of the orthogonal invariance of the standard Gaussian distribution.  We present two elementary proofs.  The first proceeds by showing that $Y_k$ is a sufficient statistic for estimating $X$ from $(Y_1, \dots, Y_k)$. The second shows that the time-reversed process is a Markov chain with independent Gaussian increments.
We denote the probability density function of $\normal(0,\delta_k \Id)$ by% $\phi_k$. 
\begin{align}
\phi_k (z) \coloneqq (2 \pi \delta_k)^{-d/2} \exp \left\{-\tfrac{1}{2\delta_k} \|z\|^2\right\}. \label{eq:phik}
\end{align}

\begin{proof}[First Proof of Lemma~\ref{lem:MarkovProperty}]
Consider the difference sequence
\begin{align}
V_k \coloneqq Y_k - Y_{k-1} = \delta_k X + \sqrt{ \delta_k} N_k.
\end{align}
Given $X=x$, the joint density of $V_1, \dots, V_k$  factors as 
%\ifisit
\begin{align}
& \prod_{i=1}^k (2 \pi \delta_i)^{-d/2} \exp\left\{   \tfrac{1}{2\delta_i} \|v_i - \delta_i x\|^2\right\} =\phi_1(v_1) \cdots \phi_k(v_k)
   \exp\left\{    \left\langle  {\textstyle \sum_{i=1}^k} v_i, x  \right \rangle  - \tfrac{t_k}{2} \|x\|^2 \right\}.
\end{align}
% \else
% \begin{align}
%  \prod_{i=1}^k (2 \pi \delta_i)^{-d/2} \exp\left\{   \tfrac{1}{2\delta_i} \|v_i - \delta_i x\|^2\right\}
% &  = \phi_1(v_1) \cdots \phi_k(v_k)
%    \exp\left\{    \left\langle  {\textstyle \sum_{i=1}^k} v_i, x  \right \rangle  - \tfrac{t_k}{2} \|x\|^2 \right\}.
% \end{align}
% \fi
By the Fisher-Neyman factorization theorem, it follows that $Y_k = \sum_{i=1}^k V_i$ is a sufficient statistic for estimating $X$ from the observations $(V_1, \dots, V_k)$.  Sufficiency also holds with respect $(Y_1, \dots, Y_k)$ which can be obtained from a one-to-one transformation of $(V_1, \dots, V_k)$.

This sufficiency implies that the parameter $X$ and the observations $(Y_1, \dots, Y_{k-1})$ are conditionally independent given the sufficient statistic $Y_k$.
The Markov property follows from combining this with the fact that $N_{k+1}$ is independent of $(Y_1, \dots, Y_k)$ and concluding that $Y_{k+1}$ and $(Y_1, \dots, Y_{k-1})$ are conditionally independent given $Y_k$. 
\end{proof}

\begin{proof}[Second Proof of Lemma~\ref{lem:MarkovProperty}] Consider the difference sequence
% \ifisit 
% \begin{align}
% B_k
% & \coloneqq \sqrt{\tfrac{t_{k+1}}{t_{k}}} Y_k - \sqrt{\tfrac{t_k}{t_{k+1}}} Y_{k+1} 
%  = \sqrt{\tfrac{t_{k+1}}{t_{k}}} W_k - \sqrt{\tfrac{t_k}{t_{k+1}}} W_{k+1}. \\[-10mm]
% \end{align}
% \else
\begin{align}
B_k
& \coloneqq \sqrt{\frac{t_{k+1}}{t_{k}}} Y_k - \sqrt{\frac{t_k}{t_{k+1}}} Y_{k+1}   = \sqrt{\frac{t_{k+1}}{t_{k}}} W_k - \sqrt{\frac{t_k}{t_{k+1}}} W_{k+1}.
\end{align}
%\fi
The second step, which follows from \eqref{eq:Ykalt},  shows that the sequences $\{B_k\}$ and $\{W_k\}$ are jointly Gaussian and independent of $X$. Using the fact that $\cov(W_k, W_m)= \min\{t_k, t_m\} \Id$, a simple calculation reveals that $ \cov(B_k, W_m) =0$ for $m >k$  
%\begin{align}
%%   \cov(B_k) & = \delta_k \\
%   \cov(B_k, W_m) &= 0 \quad \text{for all $m >  k$}, 
%\end{align}
and thus $B_k$ is independent of $( W_{k+1}, W_{k+2}, \dots )$. Putting everything together, we conclude that
\begin{align}
Y_{k} & = \frac{t_k}{ t_{k+1}} Y_{k+1}  + \sqrt{\frac{t_k}{ t_{k+1}}}  B_k, \label{eq:Yk_reverse}
\end{align}
where $B_k \sim \normal(0,\delta_{k+1} \Id)$ is independent of $(Y_{k+1}, Y_{k+2}, \dots )$. Hence,  the time-reversed process is a Markov chain with independent Gaussian increments. 
\end{proof}

%\subsection{Divergence Between Markov Chains}

Having established the Markov property, we can now provide an exact expression for $\Delta_n$ in terms of the mean-squared error $\ex{ \|X - f_k(Y_{k-1})\|^2}$  and the mutual information 
$    I(X; Y) = D( \mathsf{Law}(X,Y) \, \| \, \mathsf{Law}(X) \otimes \mathsf{Law}(Y)) $.

\begin{lemma}\label{lem:DYZexact} Under the assumptions of Theorem~\ref{thm:DYZ},
% \ifisit
% \begin{align}
% \Delta_n & =  \sum_{k=1}^n \frac{\delta_k}{2} \ex{ \| X- \ex{X \mid Y_{k-1}} \|^2}  - I(X; Y_n) \\
% & \quad +  \sum_{k=1}^n \frac{\delta_k}{2} \ex{ \| f_k(Y_{k-1}) - \ex{ X \mid Y_{k-1}} \|^2} . 
% \end{align} 
% \else
\begin{align}
\Delta_n & =  \sum_{k=1}^n \frac{\delta_k}{2} \ex{ \| X- \ex{X \mid Y_{k-1}} \|^2}  - I(X; Y_n)  +  \sum_{k=1}^n \frac{\delta_k}{2} \ex{ \| f_k(Y_{k-1}) - \ex{ X \mid Y_{k-1}} \|^2} . 
\end{align} 
%\fi
\end{lemma}
\begin{proof} 
Using the fact that both $\{Y_k\}$ and $\{Z_k\}$ are Markov chains, we can write 
\begin{align}
\Delta_n & = \sum_{k=1}^n\ex*{ \log \frac {p_k(Y_k \mid Y_{k-1})}{q_k(Y_k \mid Y_{k-1})}}, \label{eq:Delta_pq}
\end{align}
where $p_k$ and $q_k$ are the transition probability densities for $Y_k \mid Y_{k-1}$ and $Z_k \mid  Z_{k-1}$ with respect to Lebesgue measure.  From \eqref{eq:Zk}, we see that $q_k(y \mid y') = \phi_k(y - y' - \delta_k f_k(y') )$ and thus 
 \begin{align}
  -\ex*{ \log  q_k(Y_k \mid Y_{k-1})}  - \frac{d}{2} \log( 2 \pi e \delta_k)  
 &= \frac{1}{2 \delta_k} \ex*{ \| Y_k - Y_{k-1}  - \delta_k f_k(Y_{k-1}) \|^2}- \frac{d}{2}\\
  &   = 
  \frac{1}{2 \delta_k} \ex*{ \|  \delta_k X  + \sqrt{\delta_k} N_k - \delta_k f_k(Y_{k-1}) \|^2} - \frac{d}{2}  \\
  & =  \frac{\delta_k}{2 } \ex*{ \| X -  f_k(Y_{k-1}) \|^2}\\
  %&  = \frac{1}{2 \delta_k} \ex*{ \|  \delta_k X - \delta_k \ex{ X \mid Y_{k-1}} + \delta_k \ex{ X \mid Y_{k-1}}  - \delta_k f_k(Y_{k-1}) + \sqrt{\delta_k} N_k \|^2} - \frac{d}{2}  \\
 & = \frac{\delta_k}{2} \ex*{ \| X - \ex{ X \mid Y_{k-1}}\|^2 }+ \frac{\delta_k}{2 } \ex*{ \|  \ex{X \mid Y_{k-1} }  - f_k(Y_{k-1}) \|^2},
\end{align}
where the second step follows from \eqref{eq:Yk}, the third step holds because  $N_k$ is independent of $(X, Y_{k-1})$, 
and the last step follows from the orthogonality principle for conditional expectation. 

Meanwhile, noting that $Y_k \mid X, Y_{k-1}$ is Gaussian with variance $\delta_k \Id$, we see that its conditional differential entropy satisfies $h(Y_k \mid X, Y_{k-1}) = \frac{d}{2} \log(2 \pi e \delta_k)$. Accordingly,   
\begin{align}
 \ex*{ \log p_k(Y_k \mid Y_{k-1})}  + \frac{d}{2} \log (2 \pi e \delta_k) 
& =h(Y_k \mid X, Y_{k-1})  - h(Y_k \mid Y_{k-1}) \\
 & = I(X; Y_k \mid Y_{k-1}) =I(X; Y_k) - I(X; Y_{k-1}) 
\end{align}
where the last step holds because  $Y_{k-1} - Y_k - X$ is a Markov chain. Plugging these expressions back into \eqref{eq:Delta_pq} and noting that $I(X;Y_0)=0$ gives the stated result.
\end{proof}

%\subsection{Proof of Theorem~\ref{thm:DYZ} via I-MMSE}
\begin{proof}[Proof of Theorem~\ref{thm:DYZ}]
 For a distribution $\target$ on $\bbR^d$ with finite second moments, we define the MMSE function $M\colon \bbR_+ \to \bbR$ 
 \begin{align}
 M(s) &\coloneqq \ex*{ \| X - \ex{ X \mid \sqrt{s} X + N}\|^2},
\end{align}
where $X \sim \target$ and $N \sim \normal(0, \Id)$ are independent. This function is non-increasing with $M(0) = \gtr(\cov(X))$. 
The I-MMSE relation \cite{guo:2005a} states that, for any $0 \le a < b$, 
\begin{align}
I(X; \sqrt{b}X +N) - I(X; \sqrt{a} X + N) = \frac{1}{2} \int_a^b M(s) \, ds . \label{eq:I-MMSE}
\end{align}
In other words, one half the MMSE is equal to the derivative of the mutual information with respect to $s$. 

From the invariance of mutual information to one-to-one-transformation we can write $I(X;Y_k)= I(X ; t_k X + W_k) = I(X ; \sqrt{t_k} X + N)$.  Then, by the I-MMSE relation and the monotonicity of the MMSE 
(see Figure~\ref{fig:Delta_Integral}) we obtain the  sandwich 
\begin{align}
   \frac{\delta_k}{2} M(t_{k})  \le  I(X; Y_{k}) - I(X; Y_{k-1}) \le \frac{\delta_k}{2} M(t_{k-1}), \label{eq:I-MMSE_bnds}
\end{align}
Using \eqref{eq:I-MMSE_bnds},  we can now write 
\begin{align}
\sum_{k=1}^n & \frac{\delta_k}{2} \ex{ \| X- \ex{X \mid Y_{k-1}} \|^2}  - I(X; Y_n) \\
& = \sum_{k=1}^n \frac{\delta_k}{2} M(t_{k-1})  - I(X; Y_n) \\
&  = \sum_{k=1}^n \left[ \frac{\delta_k}{2} M(t_{k-1})   +I(X; Y_{k-1}) - I(X; Y_{k}) \right]  \\
& \le \sum_{k=1}^n  \frac{\delta_k}{2}  (M(t_{k-1}) - M(t_{k}) ) \\
&\le \frac{\delta_\mathrm{max}}{2} \sum_{k=1}^n  (M(t_{k-1}) - M(t_{k}) )  \\ 
& =  \frac{\delta_\mathrm{max}}{2}  (M(0) - M(t_{n}) )
\le \frac{\delta_\mathrm{max}}{2} \gtr(\cov(X)).
\end{align}
Combining with Lemma~\ref{lem:DYZexact} completes the proof.% of Theorem~\ref{thm:DYZ}.
\end{proof}

\section{Sampling Process: Intuition and  Connections}\label{sec:intuition}
The Markov property implies
that the sequence $\{Y_k\}$ can be generated by sampling each $Y_k$ conditionally given $Y_{k-1}$.   This procedure can be implemented using the following steps: 
\begin{enumerate}
    \item Draw $X_k$ from 
    the conditional of $X$ given $Y_{k-1}$;
    \item Draw $\Np_k\sim \normal(0, \Id)$ independently of $(X_k, Y_{k-1})$;
    \item Set $
    Y_k = Y_{k-1}  + \delta_k X_k  + \sqrt{\delta_{k}}\Np_k$.
\end{enumerate}
%Viewed in terms of this sampling procedure, 
Hence, the process $\{Y_t\}$  can be expressed as 
\begin{align}
Y_k = Y_{k-1}  + \delta_k X_k + \sqrt{\delta_k} \Np_k \label{eq:Yk_samp}
\end{align}
We emphasize that the above procedure defines a process with exactly the same distribution as the one defined by \eqref{eq:Yk}. In both cases, the sequences are driven by standard Gaussian processes $\{N_k\}$ and $\{\Np_k\}$, but there are also some key differences: 
\begin{itemize}
 \item In \eqref{eq:Yk} the innovation term $X$ is the same for every step. Conditional on $X$, the noise $N_k$ is independent of increments $\delta_{m} X + \sqrt{\delta_m} N_m$ for all $m \ne k$.  
 \item In \eqref{eq:Yk_samp} the innovation term $X_k$ changes with each step. Conditional  on  $X_k$, the noise $\Np_k$ is independent of  $\delta_{m} X_k + \sqrt{\delta_m} \Np_m$ for $m < k$, but not for $m > k$. 
\end{itemize}

This dual representation  of the same process provides the underlying intuition for the sampling procedure. On the one hand, the original representation in \eqref{eq:Yk} verifies that $t_k^{-1} Y_k$ is distributed according to $\target \ast \normal(0, t^{-1}_k \Id)$, and thus constitutes and approximate sample from $\target$ provided that $t_k$ is large enough. But this representation does not provide any guidance on how to produce the original sample $X$. 

On the other hand,  \eqref{eq:Yk_samp} shows that the same process can be implemented by replacing the innovation $X$ with a random term  $X_k$ that depends only on $Y_{k-1}$ and some additional randomness that is independent of $(Y_1, \dots, Y_{k-1})$.  

From this point of view, the behavior of the sampling scheme in \eqref{eq:Zk} is best understood by comparing with the sampling representation in \eqref{eq:Yk_samp}. Specifically, one can view the function $f_k(Y_{k-1})$ as providing a first-order approximation to $X_n$.
%\textcolor{red}{After matching the noise term $N_k = \Np_k$, the only differences between these sequences are due to the fluctuations in $X_k - f_k(Y_{k-1})$.}  
Assuming both processes are driven by the same noise sequence $\{\Np_k\}$, the only differences are due to the fluctuations in $X_k - f_k(Y_{k-1})$. Theorem~\ref{thm:DYZ} shows that if each $f_k$ provides a suitable approximation to the conditional mean estimator, then these  fluctuations are  negligible in the large-$n$ limit.

\subsection{Connection with Stochastic Localization}

Stochastic localization refers broadly to a framework for analyzing the mixing times of Markov chains~\cite{Eldan-gafa13,Chen-focs22}.
As described in \cite{Montanari-arxiv23}, the process defined by \eqref{eq:Yk} is an example of a stochastic localization scheme. To see this,  consider the measure-valued random process $\mu_1, \mu_2, \dots, $ where
\begin{align}
    \mu_k  \coloneqq \pr{ X \in \cdot  \mid  Y_{k-1}  }
\end{align}
is the conditional distribution of $X$ given $Y_{k-1}$. As $t_k$ increases, this sequence converges to a point-mass distribution centered at some point $X_{\infty}$. Since $\ex{ \target_k} = \target$ for all $k$,  the limit $X_{\infty}$ is a sample from $\target$.

\subsection{Connection with Diffusion Models}

Within the literature (e.g., see~\cite{Sohl-icml15,song:2019generative,Ho-neurips20,Song-iclr21}), diffusion-based sampling is often described in terms of a  ``forward model'' and a ``backward model'' for an underlying diffusion process: 
\begin{itemize}
\item The \emph{forward model} starts with a sample from the target distribution (or more generally an approximation of the form $\target \ast \normal(0, T^{-1} \Id)$) 
and then incrementally transforms it into a sample from a Gaussian distribution by adding noise and rescaling. 

\item The \emph{backward model} starts with a sample from Gaussian noise and then incrementally transforms it into a sample from the target distribution via a process that combines additive noise with nonlinear transformations. 
\end{itemize}

To connect these ideas with the results in this paper, observe that the particular choice for the comparison process in \eqref{eq:Yk} implicitly defines  forward and backward models for the diffusion limits of the  sampling scheme. In particular, the decomposition in  \eqref{eq:Yk_reverse} shows that the \emph{time-reversal} of \eqref{eq:Yk}, i.e. the stochastic process given by $Y_{n} , Y_{n-1}, \dots, Y_{1}$,  can be implemented by scaling and adding independent Gaussian noise. Under appropriate rescaling, this process can be viewed as the time-discretization of an underlying diffusion process  (the implied forward model) that transitions from $\target \ast \normal(0, t_n^{-1} \Id)$  to a much noisier version $\target \ast \normal(0, \delta_1^{-1}  \Id)$.

Likewise, the sampling representation of the  comparison process \eqref{eq:Yk_samp} can be viewed as the time-discretization of a continuous-time process (the implied backward model) that transitions the noise to a target sample. By Theorem~\ref{thm:DYZ}, we see that this backward model coincides with the continuous-time limit of \eqref{eq:Zk} as $n \to \infty$ with $\delta_k = T/n$. 

\begin{remark} 
The diffusion limits for the sampling scheme considered in this paper coincide with the original formulation of stochastic localization \cite{Eldan-gafa13,Chen-focs22}. By contrast,  much of the recent work on diffusion sampling \cite{song:2019generative,Ho-neurips20,Song-iclr21} considers a different but closely related setting where the forward model is defined by the Ornstein–Uhlenbeck process. In practice, the first-order discrete-time approximations for these two settings often have the same functional form (characterized by some transformed version of \eqref{eq:Zk}), even though the underlying processes do  not have the same distribution. Further details are provided in Appendix~\ref{sec:connection_diffusion}.
\end{remark}

\section{Additional Results}\label{sec:further_results}

\subsection{Optimization of Step Sizes}

In this section we assume that each $f_k$ is the conditional mean estimator and study the dependence on  $(n,T)$ for time steps given by 
\begin{align}
    t_k = \frac{ \alpha^k - 1}{ \alpha^n  - 1}  T , \qquad k=1,\dots, n  \label{eq:tk_alpha}
\end{align}
where $\alpha > 0$ is rate parameter. 
Under this specification, the step sizes satisfy $\delta_{k+1} = \alpha \, \delta_k$ for $k \ge 1$. The case $\alpha =1$ corresponds to uniform increments, i.e., $\delta_k = T/n$.

The bound in Theorem~\ref{thm:DYZ} depends on the maximum step size $\delta_{\mathrm{max}} \ge T/n$, and this results in a linear dependence on $T$. Using a refined analysis adapted to \eqref{eq:tk_alpha}, we show that this can be improved to a poly-logarithmic dependence without any additional assumptions on $\target$.  

Let $I \colon \bbR_+ \to \bbR$ be defined according to 
\begin{align}
    I(s) \coloneqq  I(X; \sqrt{s} X + N) = \frac{1}{2} \int_0^s M(t) \, dt
\end{align}
where $X \sim \target$ and $N \sim \normal(0, \Id)$ are independent. By  the I-MMSE relation, $I$ is strictly increasing and concave. If $\target$ has finite entropy then $I(s)$ is bounded. Otherwise, $I(s)$ increases without bound. It is well known that $I(s) \le \frac{1}{2} \log \det( \Id + s \cov(X))$ with equality if and only if $\target$ is Gaussian. 

\begin{theorem}\label{thm:alpha_bound}
%Assume that %, for  $T > 0$ and $n \in \bbN$, 
If the step sizes are given by \eqref{eq:tk_alpha} and each $f_k$ is the conditional mean estimator of $X$ given $Y_{k-1}$,  then
\begin{align}
\Delta_n 
& \le (\alpha -1) \left(  \frac{T(M(0) - M(T))}{2 (\alpha^n -1) } + I(T) - \frac{ T M(T)}{2} \right)
\end{align}
\end{theorem}

To help interpret this result, observe that the limit as $\alpha \to 1$ gives a bound for uniform step sizes: $\Delta_n \le \frac{1}{2n} T (M(0) - M(T))$. This bound,  which is essentially the same bound as in Theorem~\ref{thm:DYZ}, has a linear dependence on $T$.  More generally, optimizing over $\alpha$ as a function of the pair $(n,T)$ can lead to significant improvements. For example, the following corollary shows that that $T$ can scale nearly exponentially with $n$ with a negligible impact on the convergence rate. 

\begin{cor}
Let $T_n$ be a sequence satisfying $T_n \to \infty$ and $\frac{1}{n} \log T_n \to 0$. If $\alpha_n = (T_n \log T_n)^{1/n}$ then 
\begin{align}
    \Delta_n \le  \frac{\log(T_n) I(T_n)}{n} (1 + o_n(1)).
\end{align}
Moreover, combining with the upper bound $I(s) \le \frac{d}{2} \log(1 + \frac{s}{d} M(0))$ yields
\begin{align}
    \Delta_n \le  \frac{d (\log T_n)^2}{2 n} (1 + o_n(1)).
\end{align}
\end{cor}
\begin{proof} Starting with Theorem~\ref{thm:alpha_bound} and dropping the negative terms gives the simplified bound:
\begin{align}
\Delta_n  & \le 
 (\alpha_n -1) \left(  \frac{T_n M(0)}{2 (\alpha_n^n -1) } + I(T_n) \right)
\end{align}
Under the assumptions on $T_n$, it is easily verified that $\alpha_n - 1  =  \frac{1}{n} \log(T_n) (1  + o_n(1))$ and $T_n / (\alpha^n_n - 1)  = T_n/(T_n \log(T_n) - 1)\ = o_n(1)$ as $n \to \infty$. 
\end{proof}
\color{black}

\begin{proof}[Proof of Theorem~\ref{thm:alpha_bound}]
Starting with Lemma~\ref{lem:DYZexact}, and using the fact that $\delta_{k} = \alpha\,  \delta_{k-1}$ for $k \ge 2$ leads to 
\begin{align}
\Delta_n & =  \sum_{k=1}^n \frac{\delta_k}{2}  M(t_{k-1})  - I(t_n) \\
& =  \frac{\delta_1}{2}   M(t_0) +   \sum_{k=2}^{n} \frac{\delta_{k}}{2} M(t_{k-1})  - I(t_n)\\
& =  \frac{\delta_1}{2} M(t_0) +  \alpha\sum_{k=2}^{n} \frac{\delta_{k-1}}{2}  M(t_{k-1})  - I(t_n)\\
& = \frac{\delta_1}{2} M(t_0) - \frac{\alpha\delta_n}{2}  M(t_n) + \alpha \sum_{k=1}^{n} \frac{\delta_{k}}{2} M(t_{k})  - I(t_n)
\end{align}
By the I-MMSE inequality in \eqref{eq:I-MMSE_bnds}, we have
\begin{align}
    \sum_{k=1}^{n} \frac{\delta_{k}}{2} M(t_{k}) \le  \sum_{k=1}^n [ I(t_k) - I(t_{k-1}) ]  = I(t_n) 
\end{align}
and this leads to 
\begin{align}
\Delta_n 
& \le  \frac{\delta_1}{2}  M(t_0) - \frac{\alpha\delta_n }{2}  M(t_n)  + (\alpha -1)I(t_n).
\end{align}
Noting that $\delta_1 =\frac{ \alpha -1}{ \alpha^n - 1} T$ and $\delta_n  = (\alpha -1) t_n +  \frac{ \alpha -1}{ \alpha^n - 1} T$  completes the proof. 
\end{proof}

\subsection{Improved Rates via Moment Matching}

The sampling scheme in \eqref{eq:Zk} provides a deterministic approximation $f_{k}(Y_{k-1})$ to the conditional samples $X_k$ appearing \eqref{eq:Yk_samp}. With respect to the relative entropy, the optimal approximation is the one that matches the mean. 

More generally, our analysis extends naturally to sampling schemes that replace the function evaluation $f_k(Y_{k-1})$ with a stochastic approximation to $X_k$. The only requirement, is that the resulting process is a Markov chain.  Specifically, we consider the generalization of \eqref{eq:Zk} given by
\begin{align}
    Z_k = Z_{k-1} + \delta_k \hat{X}_k  + \sqrt{\delta_k} \Np_k ,  \label{eq:Zkalt}
\end{align}
where $\hat{X}_k$ is sampled conditionally on $Z_{k-1}$ according to a Markov kernel $Q(\cdot \mid z,t)$ evaluated at $(Z_{k-1}, t_{k-1})$. Note that  by \eqref{eq:Yk_samp}, the process $\{Y_k\}$ corresponds to the Markov kernel $ P(\cdot \mid y,t )  = \pr{ X \in \cdot \mid Y(t) = y} $ where $ Y(t) \coloneqq t X + \sqrt{t}  N $ with  $X \sim \mu$ and $N \sim \normal(0, \Id)$ are independent.

\begin{condition}\label{cond:moments} For $T > 0$ and $m \in \bbN$, 
\begin{enumerate}[label = \arabic*), leftmargin=0.5cm]
\item \textbf{Sub-Gaussian Tails:} 
There exists $L > 0$ such that
\begin{align}
 \int e^{ \|x\|^2/L^2}  \mu(dx) \le 2     \quad \text{and} \quad 
 \int e^{ \|x\|^2/ L^2}  \ex*{Q( dx \mid Y(t), t)  } &\le 2    , \quad \forall t \in [0,T]
\end{align}
where $Y(t) = t X + \sqrt{t} N$ with %(X, N) \sim \mu \otimes \normal(0, \Id)$.
$X \sim \mu$ and $N \sim \normal(0, \Id)$ independent. 

\item  \textbf{Matched Moments:} 
For all $\alpha \in \bbN_0^d$ with $\sum_{i=1}^d \alpha_i \le m$,  
    \begin{align}
        \int x^{\alpha} P(dx \mid y,t) =\int x^{\alpha} Q(dx \mid y,t) 
    \end{align}
for all $(y,t) \in \bbR^d \times [0,T]$ where $x^\alpha \coloneqq \prod_{i=1}^d x_i^{\alpha_i}$.
\end{enumerate}
    
\end{condition}

The following theorem shows that matching higher-order moments can increase the rate of convergence in terms of the number of time steps.% \textcolor{red}{; see the extended version \cite{reeves:2025information} for proof.} \nocite{chen:2022asymptotics}

\begin{theorem} \label{thm:moments}
Let $\{Z_k\}$ be generated according to \eqref{eq:Zkalt}. Under Condition~\ref{cond:moments},  there exists a positive constant $c_{d,m}$  depending only on $(d,m)$ such that 
\begin{align}
  \Delta_n \le c_{d,m} L^{2(m+1)}  \sum_{k=1}^n \delta_k^{m+1} .
\end{align}
In particular, if $\delta_k = T/n$, then 
\begin{align}
    \Delta_n \le c_{d,m,T,L} \,  n^{-m} .
\end{align}
\end{theorem}

\begin{remark}
The case of matched second moments  ($m =2$)  can be realized as a modification of  \eqref{eq:Zk} that also adapts the covariance of the Gaussian noise term, i.e.,
\begin{align}
    Z_k = Z_{k-1} + \delta_k f_k(Z_{k-1})  + \left( \delta_k^2 \Sigma_k(Z_{k-1}) + \delta_k \Id \right)^{1/2}  N_k  \label{eq:Zkm2}
\end{align}
where $f_k(y) = \ex{X \mid Y_{k-1} =y}$ and $\Sigma_k(y) = \cov(X \mid Y_{k-1} = y)$ are the conditional mean and covariance functions defined with respect to \eqref{eq:Yk}.
By construction, the implied Markov kernel for this process satisfies the matched moments condition for $m = 2$.
If $\{f_k\}$ and $\{\Sigma_k\}$ can be approximated accurately, the divergence decreases at rate $1/n^2$. This analysis is related to acceleration schemes proposed in \cite{wu2024:stochastic,li:2024provable}.
\end{remark}

%\ifisit
% \begin{proof}[Proof Sketch of Theorem~\ref{thm:moments}] Due to space constraints, we outline the key steps in the proof.  %Let $Y(t)$ be given by \eqref{eq:Yt}
% %Let $\{W(t) :t \ge 0\}$ be a continuous-time Gaussian process with covariance 
% %For $t \in \bbR_+$ set $Y(t) = t X + \sqrt{t} N $ where $X \sim \target$ and $N \sim \normal(0,\Id)$ are independent 
% Define the conditional distributions: 
% \begin{align}
% \mu_t &= P(\cdot \mid  Y(t), t)\quad \text{and} \quad \nu_t  = Q(\cdot \mid  Y(t), t)
% \end{align}
% Using the fact that both $\{Y_t\}$ and $\{Z_t\}$ are Markov chains along with the fact that relative entropy is invariant to one-to-one transformations, we can write
% \begin{align}
%  \Delta_n &  =   \sum_{k=1}^n   \ex{D\big( \mu_{t_{k-1}} \ast  \normal(\delta_k^{-1})  \, \| \, \nu_{t_{k-1}} \ast  \normal(\delta_k^{-1}\big) }  \label{eq:Delta_moments}
% \end{align}
% where we have introduced the notation $\normal(u)= \normal(0, u \Id)$. 

% Combining Condition~\ref{cond:moments}
% with Inequalities (7.3), (7.5) and Lemma~3.2, in \cite{chen:2022asymptotics} one can establish that 
% \begin{align}
%      \limsup_{\gamma \to \infty} \sup_{t \in [0,1]} \gamma^{m+1} \ex{D( \mu_t \ast  \normal(\gamma) \, \| \, \nu_t \ast  \normal(\gamma) ) }  < \infty  
%  \end{align}
% Combining with \eqref{eq:Delta_moments} completes the proof. 
% \end{proof}
%\else

\begin{proof}[Proof of Theorem~\ref{thm:moments}] 
Define the conditional distributions: 
\begin{align}
\mu_t &= P(\cdot \mid  Y(t), t)\quad \text{and} \quad \nu_t  = Q(\cdot \mid  Y(t), t)
\end{align}
where $Y(t) = t X + \sqrt{t} N $ with $X \sim \target$ and $N \sim \normal(0,\Id)$ independent.  Using the fact that both $\{Y_k\}$ and $\{Z_k\}$ are Markov chains along with the fact that relative entropy is invariant to one-to-one transformations, we can write
\begin{align}
 \Delta_n &  =   \sum_{k=1}^n   \ex{D\big( \mu_{t_{k-1}} \ast  \normal(\delta_k^{-1})  \, \| \, \nu_{t_{k-1}} \ast  \normal(\delta_k^{-1}\big) }  \label{eq:Delta_moments}
\end{align}
where we have introduced the notation $\normal(u)= \normal(0, u \Id)$.

By the moment matching assumption in  Condition~\ref{cond:moments} and Lemma~\ref{lem:Dmoment} in Appendix~\ref{sec:moment_matching_bnd}.  there is a constant $c_{d,m}$ such that, for any $\beta > 0$,  the following holds almost surely: 
\begin{align}
&D(\mu_{t_{k-1}} \ast \normal(\delta_k^{-1}) \,\| \,  \nu_{t_{k-1}} \ast \normal(\delta_k^{-1})) \\
& \le  c_{d,m}  \,  (\delta_k/\beta) ^{m+1} \left( \int e^{  \beta \|x\|^2} P(dx \mid Y(t_{k-1}), t_{k-1})   +\int e^{  \beta \|x\|^2} Q(dx \mid Y(t_{k-1}), t_{k-1})  \right).
\end{align}
Evaluating with $\beta = 1/L^2$, taking the expectation of both sides and invoking the sub-Gaussian tail assumption in Condition~\ref{cond:moments} then yields
\begin{align}
 \ex{ D(\mu_{t_{k-1}} \ast \normal(\delta_k^{-1}) \,\| \,  \nu_{t_{k-1}} \ast \normal(\delta_k^{-1})) } 
& \le 4c_{d,m} L^{2(m+1)} \delta_k^{m+1} 
\end{align}
Plugging this inequality back into \eqref{eq:Delta_moments} completes the proof. 
\end{proof}

%\fi

%\clearpage

%\bibliographystyle{IEEEtran}
%\bibliography{long_names,library,henry,galen} 

%\end{document} 

\appendix

\section{Moment Matching Divergence Bound}
\label{sec:moment_matching_bnd}

The following result provides a uniform upper bound on the divergence between distributions satisfying a moment matching condition. The proof is adapted from the proof of Theorem~2.5 in \cite{chen:2022asymptotics}, which provides the exact asymptotics in the $s \to 0$ limit.

\begin{lemma}\label{lem:Dmoment}
Let $\mu$ and $\nu$ be distributions on $\bbR^d$  such that, for every $\alpha \in \bbN_0^d$ with $\sum_{i=1}^d \alpha_i \le m$, 
\begin{align}
\int x^\alpha \mu(dx) = \int x^{\alpha} \nu(dx).
\end{align}
Then, for all $s, \beta > 0$,
\begin{align}
D(\mu \ast \normal(s^{-1}) \,\| \,  \nu \ast \normal(s^{-1})) 
 \le  c_{d,m}  \,  (s/\beta) ^{m+1} \left( \int e^{  \beta \|x\|^2} \mu(dx)  +  \int e^{ \beta  \|x\|^2} \nu(dx)  \right),
\end{align}
where $c_{d,m}$ is a positive constant that depends only $(d,m)$. 
\end{lemma}
\begin{proof}
Throughout this proof, we use the notation $f \lesssim g$ to indicate that the inequality $f \le c_{d,n} g$ holds for some positive constant $c_{d,n}$ that may depend on $(d,n)$. Let  $\gamma = \normal(0, \Id)$ be the standard Gaussian measure, and  let $(A, B)$ be independent of $N \sim \gamma$  with marginals $A \sim \mu $ and $B \sim \nu$. In the following, we will prove that 
\begin{align}
D(\mu \ast \normal(s^{-1}) \,\| \,  \nu \ast \normal(s^{-1})) 
& \lesssim s^{m+1} \left( \ex[\big]{ e^{  6 \|A\|^2}} + \ex[\big]{ e^{6  \|B\|^2}}  \right), \quad \text{for all  $s > 0$}. \label{eq:Ds6}
\end{align}
For $\beta >0$, we then recover the stated inequality by observing that the relative entropy is invariant under the 
simultaneous rescaling of $s$ and $\mu, \nu$ defined by $(s, A, B) \mapsto \big(\tfrac{6}{\beta}   s, \sqrt{ \tfrac{\beta}{6}}   A,  \sqrt{ \tfrac{\beta}{6}} B\big)$. Making this change of variables and then absorbing the factor of $6^{m+1}$ into the constant gives the desired result.

In order to prove \eqref{eq:Ds6}, we first  consider the case $s \ge 1$. By the convexity of relative entropy and Jensen's inequality,
\begin{align}
D(\mu \ast \normal(s^{-1}) \,\| \,  \nu \ast \normal(s^{-1}))
& \le \ex*{ 
D( \normal(A , s^{-1}\Id) \,\| \,   \normal(B , s^{-1}\Id)) } =\frac{s}{2} \ex{ \|A- B\|^2}.
\end{align}
Combining with the basic inequality $\|A-B\|^2 \le 2 (\|A\|^2 + \|B\|^2) \le \frac{1}{3} (e^{ 6\|A\|^2} + e^{6 \|B\|^2})$ along with the fact $s \le s^{m+1}$ for all $s \ge 1$ verifies that \eqref{eq:Ds6} holds uniformly over $s \in[1, \infty)$.

Next, we consider the case $s \in (0,1)$. Let the densities of  $\sqrt{s} A + N$ and $\sqrt{s} B + N$ with respect to $\gamma$ be denoted by
\begin{align}
p_s(x) \coloneqq \ex{ e^{ \sqrt{s} \langle x, A \rangle - \frac{s}{2} \|A\|^2}}, \qquad 
q_s(x) \coloneqq \ex{ e^{ \sqrt{s} \langle x, B \rangle - \frac{s}{2} \|B\|^2}}.
\end{align}
Using the fact that relative entropy is bounded from above by the chi-square divergence, and then applying H\"older's inequality with conjugate exponents $3$ and $3/2$ gives
\begin{align}
D(\mu \ast \normal(s^{-1}) \,\| \,  \nu \ast \normal(s^{-1}))  
& \le  \int \frac{ (p_s(x) - q_s(x)^2}{ q_s(x) }  \, \gamma(dx)  \\
& \le  \Big(  \int \frac{\gamma(dx)}{q_s(x)^3}      \Big)^{1/3} \Big(\int \left| p_s(x) - q_s(x)\right|^3\, \gamma(dx)  \Big)^{2/3} . \label{eq:Ds6b}
\end{align}
For the first term, Jensen's inequality gives the lower bound
$ q_s(x) 
   \ge \exp\{ \sqrt{s} \langle x, \ex{B} \rangle - \frac{s}{2} \ex{ \|B\|^2}\}$, 
which leads to 
\begin{align}
\int \frac{ \gamma(dx)  }{q_s(x)^3}  &\le   \int e^{ - 3\sqrt{s} \langle x, \ex{B} \rangle + \frac{3s}{2}  \ex{ \|B\|^2}} \, \gamma(dx)    =  e^{  \frac{9}{2} s \|\ex{B}\|^2  + \frac{3s}{2}   \ex{ \|B\|^2}}  \le  \ex[\big]{    e^{6s \|B\|^2}} .\label{eq:Ds6c}
\end{align}
For the second term, we use the $m$-th order Taylor series expansion of $y \mapsto \exp\{ \langle x, y \rangle - \frac{1}{2} \|y\|^2\}$  about the point $y = 0$ to obtain 
\begin{align}
p_s(x)  & =  \sum_{\alpha \in \bbN_0^d \, : \, |\alpha| \le m} \frac{s^{|\alpha|/2} H_{\alpha}(x)}{ \alpha!}  \ex{ A^\alpha} + \ex{ r_{m+1}(x, \sqrt{s} A)}\\
q_s(x)  & =  \sum_{\alpha \in \bbN_0^d \, : \, |\alpha| \le m} \frac{s^{|\alpha|/2} H_{\alpha}(x)}{ \alpha!}  \ex{ B^\alpha} + \ex{ r_{m+1}(x, \sqrt{s} B)}
\end{align}
where $H_\alpha$ are the Hermite polynomials and $r_{m+1}(x,y)$ is the remainder term. The  assumption that  $A$ and $B$  have the same moments of up to order to $m$ along with the basic inequality $|u-v|^3 \le 4 (|u|^3+ |v|^3)$ then leads to 
\begin{align}
 \int \left| p_s(x)  - q_s(x) \right|^3\, \gamma(dx) 
& = \int \left|\ex{ r_{m+1}(x, \sqrt{s} A)} -\ex{ r_{m+1}(x, \sqrt{s} B)}  \right|^3\, \gamma(dx) \\
& \le    4 \int  \left( \left|\ex{ r_{m+1}(x, \sqrt{s} A)} \right|^3 +    \left|\ex{ r_{m+1}(x, \sqrt{s} B)} \right|^3 \right)  \, \gamma(dx) 
\end{align}

To proceed, we use the integral form for the remainder given in Equation (4.3) of \cite{chen:2022asymptotics},
\begin{align}
r_{m+1}(x,y) &= (m+1) \int_0^1 (1 - u)^m \sum_{\substack{\alpha \in \bbN_0^d \\  |\alpha| = n+1}} \frac{ y^\alpha}{ \alpha!}  g_{\alpha}(x, u y) \, du 
\end{align}
where $|\alpha| \coloneqq \alpha_1+ \cdots + \alpha_d$ and  $g_{\alpha}(x,y) \coloneqq H_\alpha(x - y) \exp\{ \langle x, y \rangle  - \frac{1}{2} \|y\|^2\}$. For $Y \in \{  \sqrt{s}A, \sqrt{s} B\}$, we can now write
\begin{align}
\int  \left|\ex{ r_{m+1}(x, Y)} \right|^3 \, \gamma(dx)  
& \lesssim    \max_{0 \le u \le 1}  \max_{\substack{\alpha \in \bbN_0^d \\  |\alpha| = n+1}}  \ex*{ \|Y\|^{3m+3} \int   \left|  g_{\alpha}(x ,  u Y) \right|^3 \, \gamma(dx)}  .
\end{align}
Since each $H_{\alpha}$ is a polynomial of degree $m+1$,
\begin{align}
   \left| g_{\alpha}(x,y)\right| &\lesssim  \, (1 + \|x-y\|^{m+1}) e^{ \langle x, y \rangle - \frac{1}{2} \|y\|^2},
\end{align}
and this leads to 
\begin{align}
   \int \left| g_{\alpha}(x,y)\right|^3 \, \gamma(dx) 
& \lesssim \int (1 + \|x- y\|^{3m+3}) e^{ 3 \langle x, y \rangle - \frac{3}{2} \|y\|^2} \, \gamma(dx) \\
& = \int (2\pi)^{-d/2}  (1 + \|x- y\|^{3m+3}) e^{-  \frac{1}{2} \| x - 3 y\|^2  +3 \|y\|^2} \, dx \\
& \lesssim  (1 + \|y\|^{3m+3}) e^{ 3\|y\|^2}.  
\end{align}
Combining with the display above, we see that
\begin{align}
\int  \left|\ex{ r_{m+1}(x, \sqrt{s} A )} \right|^3 \, \gamma(dx)  
& \lesssim \, s^{\frac{3}{2} (m+1)}  \ex*{ \|A\|^{3m+3} (1 + \|\sqrt{s}A\|^{3m+3}) e^{ 3s \|A\|^2}}\\
& \lesssim  \, s^{\frac{3}{2} (m+1)}  \ex[\big]{ e^{ 6 \|A\|^2}} \quad \text{ for all $0 < s \le 1$,} \label{eq:Ds6d}
\end{align}
with the same inequality holding for $B$. 

Plugging  \eqref{eq:Ds6c} and \eqref{eq:Ds6d} back into \eqref{eq:Ds6b}, we conclude that 
\begin{align}
    D(\mu \ast \normal(s^{-1}) \,\| \,  \nu \ast \normal(s^{-1})) \lesssim s^{m+1} \left( \ex[\big]{ e^{ 6 \|A\|^2}}+\ex[\big]{ e^{ 6 \|B\|^2}} \right ) \quad \text{for all $0 \le s \le 1$},\label{eq:Ds6e}.
\end{align}
Having verified both the cases $s \in [1,\infty)$ and $s \in (0,1)$ the proof of proof of \eqref{eq:Ds6} is complete. 
%and this completes the proof of \eqref{eq:Ds6}. 
%Combining \eqref{eq:Ds6a} and \eqref{eq:Ds6e} completes the proof of \eqref{eq:Ds6}. 
\end{proof}

\section{Connection with Diffusion Models} \label{sec:connection_diffusion}

As discussed in Section~\ref{sec:intuition},  diffusion-based sampling methods are often described in terms of forward and backward diffusion models. The forward model defines a process  $(\Xf_t)_{t \ge 0}$ whose distribution transitions from  the target distribution $\target$ at time $t = 0$ to the standard Gaussian distribution as $t$ increases. The backward model defines a process  $(\Xb_t)_{t \ge 0}$ that transforms a noise variable into a sample from $\mu$. 

\begin{itemize}
\item Much of sampling literature is described in terms of  the Ornstein–Uhlenbeck process (OU) and considers a parameterization that gives rise to score-based estimation. For a family of densities $(p_t)_{t \ge 0}$ on $\bbR^d$, the score function $s \colon \bbR^d \times [0,\infty)$ is the gradient of the log density $
    s(y,t) \coloneqq \nabla \log p_t(y)$.

\item Meanwhile, the formulation used in stochastic localization and also this paper is more directly connected to standard Brownian motion. In this setting the nonlinearity appearing in the discretization is the conditional mean estimator.  
 \end{itemize}

Under additive Gaussian noise,  affine mapping between the score function and the conditional mean estimator. Specifically, if $p_t$ is the density of $Y_t = a_t  X + \sigma_t N$ for scalars $(a_t, \sigma_t)$  and random variables $(X, N) \sim \mu \otimes  \normal(0,1)$ then Tweedie'sformula yields 
\begin{align}
    s(y,t) & =  \frac{a_t \ex{X \mid Y  = y} - y}{ \sigma_t^2} 
\end{align}
Consequently, the discrete-time approximations to these models are both functionally equivalent to the basic sampling process introduced in \eqref{eq:Zk}.  The theoretical guarantees depend on the choice of the comparison process.

\end{document}